\documentclass{article}

\PassOptionsToPackage{numbers, compress}{natbib}
\usepackage[preprint]{neurips_2026}

\usepackage[utf8]{inputenc} 
\usepackage[T1]{fontenc}    
\usepackage{hyperref}       
\usepackage{url}            
\usepackage{caption}
\usepackage{booktabs}       
\usepackage{multirow}
\usepackage{colortbl}
\usepackage{wrapfig}
\usepackage{amsfonts}       
\usepackage{nicefrac}       
\usepackage{microtype}      
\usepackage{xcolor}         
\hypersetup{
    citebordercolor=cyan, 
    linkbordercolor=cyan, 
    urlbordercolor=cyan   
}
\usepackage{graphicx}
\usepackage{amsmath,amssymb,bm,amsthm}
\usepackage{subcaption}
\usepackage{enumitem}
\usepackage{tcolorbox}

\newtheorem{definition}{Definition}
\newtheorem{proposition}{Proposition}
\newtheorem{lemma}{Lemma}
\newtheorem{theorem}{Theorem}

\title{Self-Distilled Trajectory-Aware Boltzmann Modeling: Bridging the Training-Inference Discrepancy in Diffusion Language Models}

\author{
  Kecheng Chen$^{1,\bigstar}$, Ziru Liu$^{2,\bigstar,\spadesuit}$, Xijia Tao$^{3}$,
  Hui Liu$^{1}$, Yibing Liu$^{1}$, Xinyu Fu$^{2}$, Shi Wu$^{2}$,\\ \textbf{Suiyun Zhang}$^{2}$,  \textbf{Dandan Tu}$^{2,\clubsuit}$,
  \textbf{Lingpeng Kong}$^{3}$, \textbf{Rui Liu}$^{2,\clubsuit}$, \textbf{Haoliang Li}$^{1,\clubsuit}$ \\
  \\
  $^{1}$City University of Hong Kong,
  $^{2}$Huawei Research,
  $^{3}$The University of Hong Kong
  \\
  Email: \texttt{tudandan@huawei.com};  \texttt{liu.rui2@huawei.com}; \texttt{ haoliang.li@cityu.edu.hk}\\
}

\begin{document}

\maketitle
\footnotetext[1]{$^{\bigstar}$ Contribute equally to this work. $^{\clubsuit}$ Corresponding authors. $^{\spadesuit}$ Project Lead.}

\begin{abstract}
  Diffusion Language Models (DLMs) have recently emerged as a promising alternative to autoregressive language models, offering stronger global awareness and highly parallel generation. 
  However, post-training DLMs with standard Negative Evidence Lower Bound (NELBO)-based supervised fine-tuning remains inefficient: training reconstructs randomly masked tokens in a single step, whereas inference follows a confidence-guided, multi-step easy-to-hard denoising trajectory. 
  Recent trajectory-based self-distillation methods exploit such inference trajectories mainly for sampling-step compression and acceleration, often improving decoding efficiency without substantially enhancing the model's underlying capability, and may even degrade performance under full diffusion decoding. In this work, we ask whether self-distilled trajectories can be used not merely for faster inference, but for genuine knowledge acquisition. 
  Although these trajectories lie on the pretrained DLM's own distributional manifold and thus offer a potentially lower optimization barrier, we find that naively fine-tuning on them with standard NELBO objectives yields only marginal gains. 
  To address this limitation, we propose \textbf{T}rajectory-\textbf{A}ligned optimization via \textbf{Bo}ltzmann \textbf{M}odeling (\textbf{TABOM}), a self-distilled trajectory-based post-training framework that aligns training with the easy-to-hard structure of inference. 
  TABOM models the inference unmasking preference as a Boltzmann distribution over predictive entropies and derives a tractable pairwise ranking objective to align the model's certainty ordering with the observed decoding trajectory. 
  Empirically, TABOM achieves substantial gains in new domains, expands the effective knowledge boundary of DLMs, and significantly mitigates catastrophic forgetting compared with standard SFT.
\end{abstract}

\section{Introduction}
\label{sec:Introduction}

\begin{figure}[t]
  \centering
  \includegraphics[width=\textwidth]{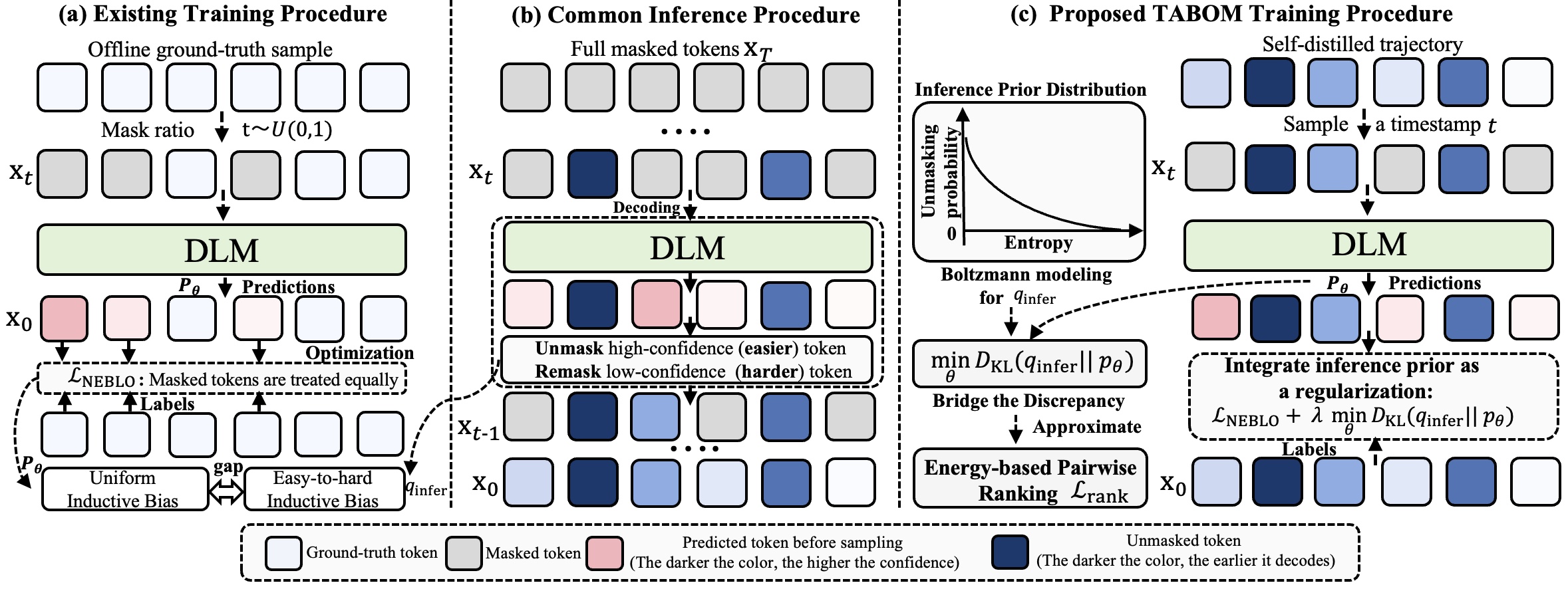}
  \caption{Overview of the proposed TABOM framework. Standard SFT trains DLMs with uniformly random masking, which induces a mismatch with the easy-to-hard unmasking trajectory used during inference. TABOM instead learns from self-distilled trajectories by combining trajectory-aware reconstruction with an energy-based pairwise ranking loss, aligning the model's entropy landscape with the inference-time Boltzmann unmasked distribution.}
  \label{fig:framework} 
  \vspace{-0.5cm}
\end{figure}

Modern Large Language Models (LLMs) have demonstrated exceptional capabilities across diverse tasks, including mathematical reasoning, coding, and multi-turn dialogue \citep{kimiteam2025kimik2openagentic,openAI-o3,deepseekai2025deepseekv3technicalreport}. 
While the prevailing paradigm relies on autoregressive transformer architectures \citep{vaswani2023attentionneed}, 
Diffusion Language Models (DLMs) have emerged as a compelling alternative \citep{ye2025dream, nie2025large}. 
By leveraging bidirectional contexts and incorporating information from all positions simultaneously, DLMs offer superior global awareness and the potential for highly parallel token generation. 

To further enhance pretrained DLMs, supervised fine-tuning (SFT) \citep{ye2025dream} is commonly applied by training the model to predict randomly masked tokens in a single forward pass. However, this objective is substantially less sample-efficient than autoregressive training and introduces a pronounced training-inference mismatch, since inference requires iterative denoising along a decoding trajectory. A natural remedy is therefore to directly learn high-quality decoding trajectories, which requires both a scalable trajectory generation pipeline and an efficient trajectory-level learning objective. 

Although reinforcement learning (RL) may seem relevant to this goal \citep{ou2025principledrldiffusionllms,yang2026darediffusionlargelanguage,zhao2025d1scalingreasoningdiffusion,yang2025tamingmaskeddiffusionlanguage}, it is not a direct solution. Outcome-reward RL provides only coarse sequence-level supervision, making credit assignment over intermediate denoising decisions difficult. It also requires repeated rollouts and reward evaluation, which is especially costly for DLMs since each rollout already involves multi-step denoising. Recent work has explored trajectory-based self-distillation for DLM post-training \citep{ma2025dinferefficientinferenceframework,song2025seeddiffusionlargescalediffusion,zhang2026t3d}, where the model learns from decoding traces generated by itself. However, existing methods are primarily designed for sampling-step compression. By learning shortcut transitions from later diffusion states to earlier ones, they mainly target \textit{inference efficiency}, allowing DLMs to decode with fewer steps while maintaining acceptable performance trade-offs, as demonstrated by Seed
Diffusion~\citep{song2025seeddiffusionlargescalediffusion} and dInfer~\citep{ma2025dinferefficientinferenceframework}. Despite these efficiency gains, recent evidence suggests (see Table 2 in \citet{zhang2026t3d} paper) that DLMs fine-tuned on these trajectories may suffer from degraded performance under full diffusion decoding, i.e., decoding one token per step \citep{zhang2026t3d}.

This observation motivates a more fundamental question: \textit{Beyond improving inference efficiency, can self-distilled trajectories enable genuine knowledge acquisition and performance gains?} Intuitively, because self-distilled trajectories are generated from the pretrained DLM's own distributional manifold, they may present a lower optimization barrier than externally constructed targets, thereby facilitating smoother absorption of new knowledge during
fine-tuning. However, our preliminary investigations show that naively fine-tuning DLMs on self-distilled trajectories with the standard Negative Evidence Lower Bound (NELBO) objective \citep{yang2026darediffusionlargelanguage} yields only marginal improvements.

To fully exploit self-distilled trajectories for performance enhancement, we propose \textbf{Trajectory-Aligned Optimization via Boltzmann Modeling} (TABOM), a novel post-training framework for DLMs. TABOM leverages the structured decoding patterns embedded in self-distilled trajectories to align the model's predictive distribution with its actual inference-time behavior. We theoretically formulate this target behavior as a Boltzmann distribution over the ideal predictive entropy of each token, which serves as a surrogate for the easy-to-hard inductive bias and effectively mitigates the training-inference discrepancy. The main contributions of this work are summarized as follows:
\begin{itemize}[leftmargin=*]
\item (Section \ref{sec:tao}) Theoretically, we demonstrate that the inference unmasked distribution under easy-to-hard decoding schedules can be explicitly modeled as a Boltzmann distribution over the ideal predictive entropy of each token. This allows us to formulate the training-inference alignment as a direct Kullback-Leibler (KL) divergence minimization problem.

\item (Section \ref{sec:pairwise_ranking}) Methodologically, we introduce a tractable surrogate objective based on Pairwise Ranking to optimize the intractable global KL divergence. This mechanism enforces local entropy gradients that strictly adhere to the easy-to-hard decoding schedule, smoothly transferring new knowledge into the model without disrupting its inherent predictive manifold.

\item (Section \ref{sec:Experiments1}) Experimentally, we show that TABOM resolves the central SFT dilemma identified in this paper: it turns self-distilled trajectories into substantial performance gains, while preserving the pretrained model's out-of-distribution capabilities and avoiding the catastrophic forgetting often induced by standard supervised fine-tuning.

\item (Section \ref{sec:tds}) We introduce Trajectory Discrimination Score (TDS) to quantify whether a model preserves token-level uncertainty differences along the decoding trajectory. Qualitative (see Figure \ref{fig:trajectory_discrimination_score}) and quantitative (see Table \ref{tab:tds_f100}) TDS results reveal that TABOM reshapes the entropy landscape toward the easy-to-hard inference bias, rather than merely reusing self-distilled samples.
\end{itemize}

\begin{figure}[t]
  \centering
  \begin{subfigure}{0.48\linewidth}
    \centering
    \includegraphics[width=\linewidth]{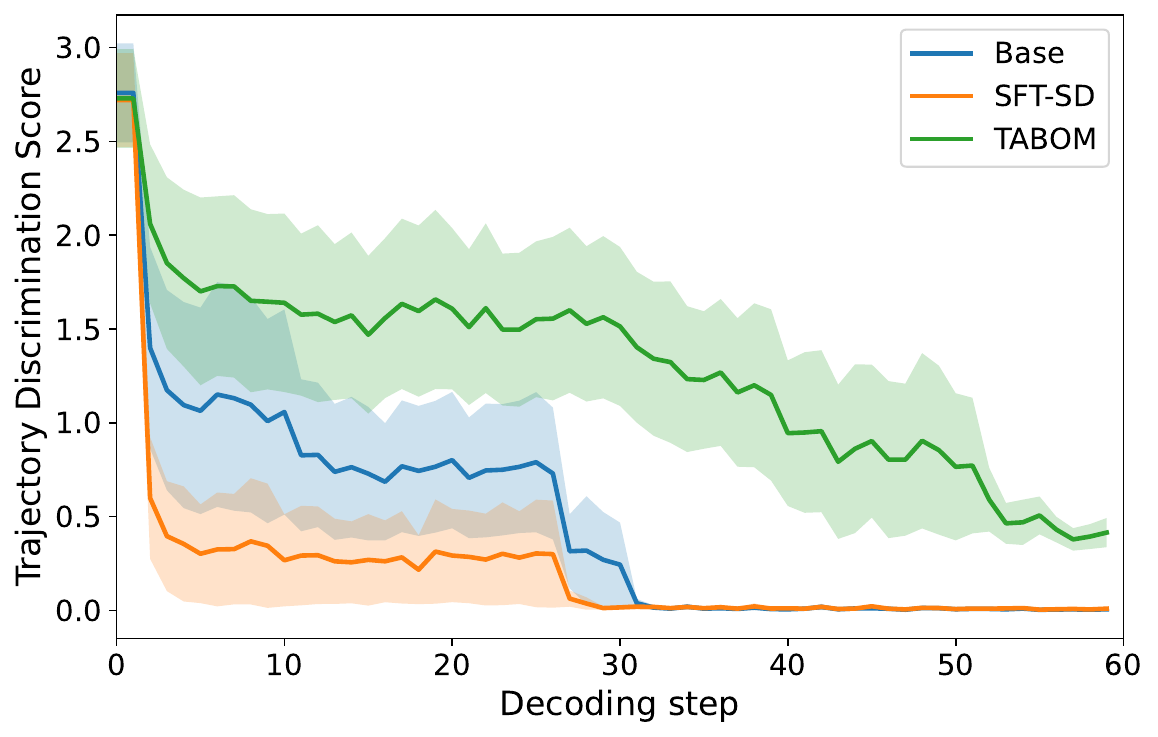}
    \caption{MBPP}
  \end{subfigure}
  \hfill
  \begin{subfigure}{0.48\linewidth}
    \centering
    \includegraphics[width=\linewidth]{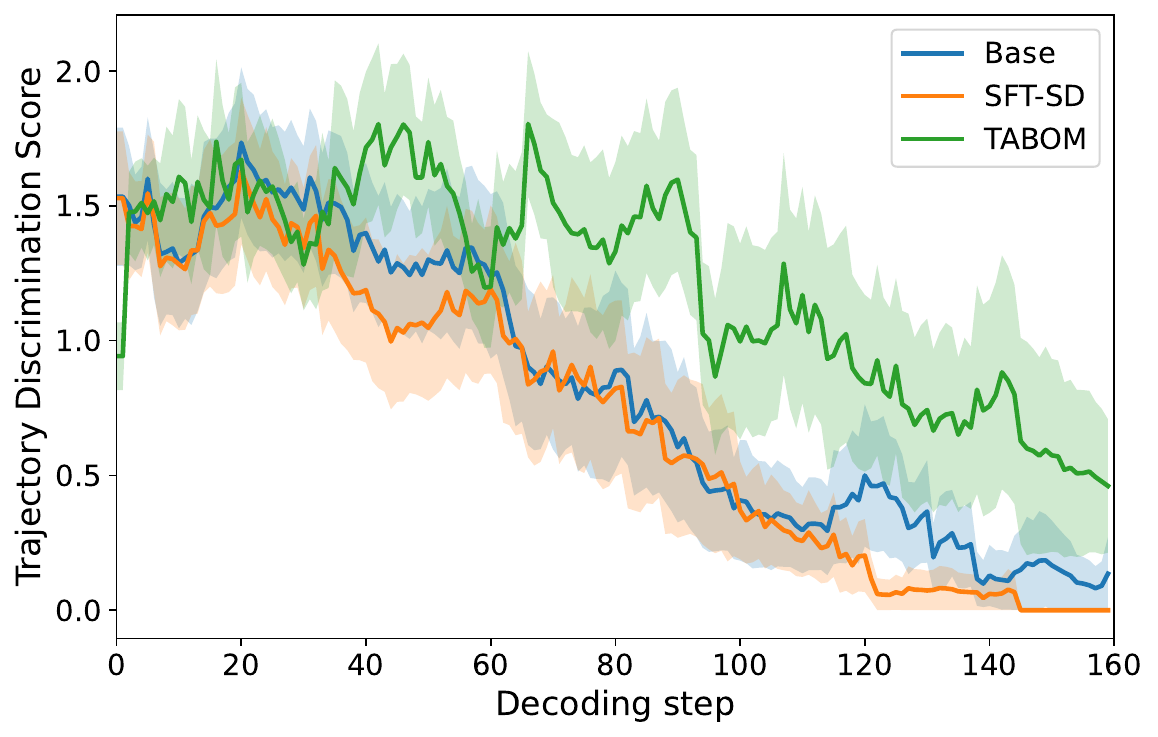}
    \caption{GSM8K}
  \end{subfigure}
  \caption{Trajectory Discrimination Score during decoding on Dream. We compute the variance of predictive entropy over decoding-time masked tokens, average it over 64 sampled trajectories, and exclude steps after the first EOS token in each trajectory. Higher curves indicate stronger token-level uncertainty discrimination along the trajectory. ``Base" denotes the original model without SFT.} 
  \label{fig:trajectory_discrimination_score}
\end{figure}

\section{Empirical Observations on Self-Distilled Trajectories}
\label{sec:obs}
\begin{wraptable}{r}{0.55\textwidth}
  \vspace{-0.4cm}
  \centering
  \resizebox{\linewidth}{!}{
  \begin{tabular}{lcccc}
  \toprule
  \textbf{Method} & \textbf{HumanEval} & \textbf{MBPP} & \textbf{GSM8K} & \textbf{MATH500} \\
  \midrule
  \multicolumn{5}{c}{\textit{Code SFT (Ling-Coder-SFT)}} \\
  \midrule
  No-SFT & 52.66 & 58.00 & 81.41 & 39.80 \\
      SFT-GT & 61.55$_{\color{green!60!black}{\mathbf{+8.89}}}$ & 58.00$_{\color{gray}{\mathbf{+0.00}}}$ & 52.33$_{\color{red}{\mathbf{-29.08}}}$ & 32.40$_{\color{red}{\mathbf{-7.40}}}$ \\
      SFT-SD & 53.66$_{\color{green!60!black}{\mathbf{+1.00}}}$ & 59.20$_{\color{green!60!black}{\mathbf{+1.20}}}$ & 81.81$_{\color{green!60!black}{\mathbf{+0.40}}}$ & 41.60$_{\color{green!60!black}{\mathbf{+1.80}}}$ \\
      TABOM (Ours) & 60.36$_{\color{green!60!black}{\mathbf{+7.70}}}$ & 60.60$_{\color{green!60!black}{\mathbf{+2.60}}}$ & 81.73$_{\color{green!60!black}{\mathbf{+0.32}}}$ & 42.40$_{\color{green!60!black}{\mathbf{+2.60}}}$ \\
      \midrule
      \multicolumn{5}{c}{\textit{Math SFT (MixChain-Z-PRM12K)}} \\
      \midrule
      No-SFT & 52.66 & 58.00 & 81.41 & 39.80 \\
      SFT-GT & 46.34$_{\color{red}{\mathbf{-6.32}}}$ & 58.00$_{\color{gray}{\mathbf{+0.00}}}$ & 80.12$_{\color{red}{\mathbf{-1.29}}}$ & 37.40$_{\color{red}{\mathbf{-2.40}}}$ \\
      SFT-SD & 57.92$_{\color{green!60!black}{\mathbf{+5.26}}}$ & 58.60$_{\color{green!60!black}{\mathbf{+0.60}}}$ & 81.95$_{\color{green!60!black}{\mathbf{+0.54}}}$ & 39.80$_{\color{gray}{\mathbf{+0.00}}}$ \\
      TABOM (Ours) & 58.54$_{\color{green!60!black}{\mathbf{+5.88}}}$ & 59.20$_{\color{green!60!black}{\mathbf{+1.20}}}$ & 84.31$_{\color{green!60!black}{\mathbf{+2.90}}}$ & 41.10$_{\color{green!60!black}{\mathbf{+1.30}}}$ \\
  \bottomrule
  \end{tabular}
  }
  \caption{Performance comparison on Dream. SFT-SD avoids catastrophic forgetting but yields marginal in-domain gains compared to SFT-GT. TABOM achieves the best of both worlds.}
  \label{tab:sft_results}
  \vspace{-0.2cm}
\end{wraptable}

To understand the impact of self-distillation on model optimization and its limitations under standard SFT paradigms,
 we conduct a comparative analysis between self-distilled (SD) data and offline ground-truth (GT) data. 
 Specifically, we conduct experiments on the Dream model \citep{ye2025dream} across two domains: code generation and mathematical reasoning. For code generation, we randomly sample 17K queries from Ling-Coder-SFT \citep{inclusionAI_ling_coder}. For mathematical reasoning, we utilize queries from MixChain-Z-PRM12K \citep{horseee_mixchain}. For each query, the offline ground-truth (GT) data consists of the standard problem-answer pairs originally provided in the datasets. In contrast, the self-distilled (SD) data is generated by having the base model decode the answers using entropy-based decoding, yielding $\sim$3.8K and $\sim$5.1K valid SD trajectories for code generation and mathematical reasoning, respectively.

\begin{figure}[h]
  \centering
  \begin{subfigure}{0.45\textwidth}
      \centering
      \includegraphics[width=\textwidth]{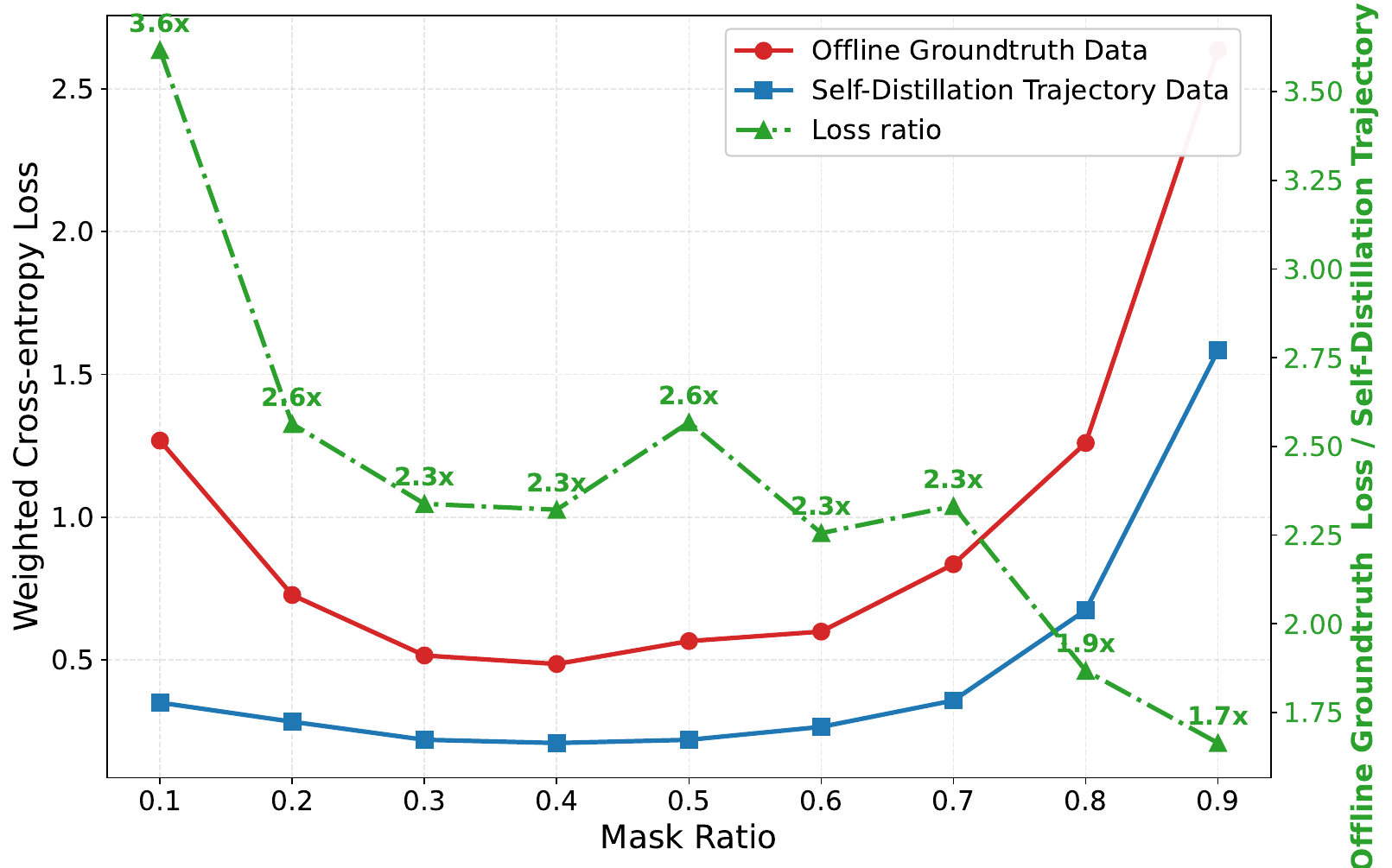}
      \caption{Code Generation}
      \label{fig:loss_code}
  \end{subfigure}
  \hfill
  \begin{subfigure}{0.45\textwidth}
      \centering
      \includegraphics[width=\textwidth]{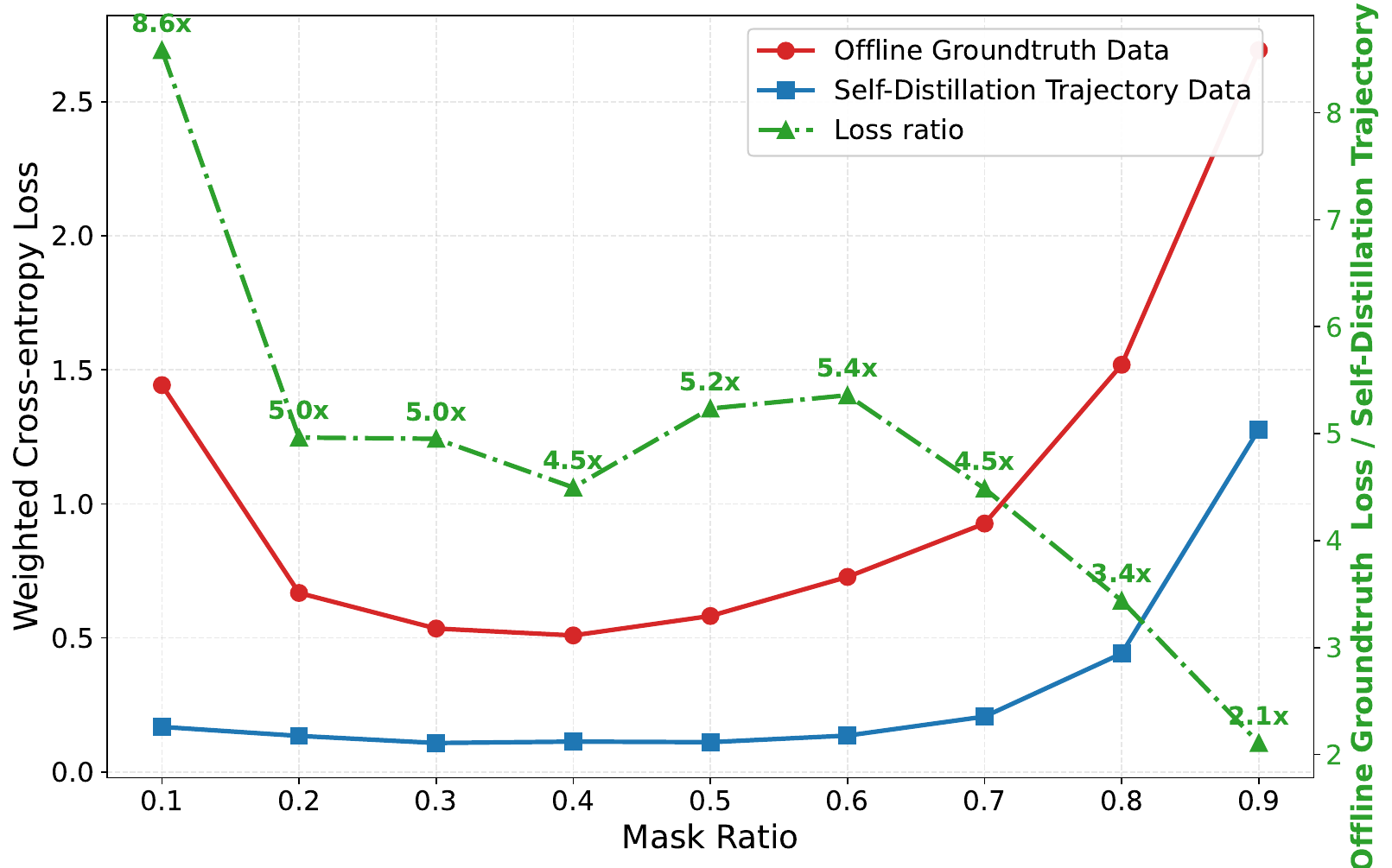}
      \caption{Math Reasoning}
      \label{fig:loss_math}
  \end{subfigure}
  \caption{Comparison of Cross-Entropy loss between GT and SD data across different mask ratios.}
  \label{fig:loss_comparison}
\end{figure}

\textbf{Lower Optimization Barrier.} First, we evaluate the model on the same queries using either GT data or SD data across various mask ratios. As illustrated in Figure \ref{fig:loss_comparison}, the cross-entropy (CE) loss for SD data is consistently lower than that of GT data. This suggests that self-distillation aligns the target distribution with the model's intrinsic predictive manifold, providing a smoother optimization landscape and lowering the optimization barrier.

\textbf{The Dilemma: Forgetting vs. Marginal Gains.} To investigate how this lower barrier translates to downstream performance, we evaluate the models fine-tuned on these datasets. As shown in Table \ref{tab:sft_results}, directly fine-tuning on GT data (SFT-GT) improves in-domain performance (e.g., HumanEval increases to 61.55) but suffers from severe catastrophic forgetting out-of-domain (e.g., GSM8K drops to 52.33). Conversely, fine-tuning on self-distilled trajectories (SFT-SD) effectively prevents catastrophic forgetting (GSM8K remains at 81.81). This is reasonable as the self-distilled trajectories are generated by the model itself, thereby preserving its inherent predictive manifold. However, the in-domain performance improvement is marginal (HumanEval with only 53.66). This empirical dilemma implies the following observation:

\begin{tcolorbox}[colback=gray!10, colframe=black, arc=4pt, boxrule=0.5pt, left=4pt, right=4pt, top=4pt, bottom=4pt]
\textit{Despite offering a lower optimization barrier, naively fine-tuning on self-distilled trajectories is insufficient to unlock substantial performance gains.}
\end{tcolorbox}

\section{Self-Distilled Trajectory-Aware Boltzmann Modeling}
\label{sec:Methodology}

\textbf{Motivation.} In this paper, we argue that the fundamental training-inference discrepancy inherent in the current NELBO-based SFT paradigm leads to marginal performance gains, despite utilizing self-distilled trajectories with a lower optimization barrier. We provide specific analysis as follows. 

Let $\overline{\mathbb{X}} = \mathbb{X} \cup \{\operatorname{M}\}$ be the extended vocabulary, where $\operatorname{M}$ denotes an absorbing mask token. Let $\mathbf{s} \in \mathbb{X}^{L}$ be a prompt sequence, 
and let $\mathbf{x}_{t} = \{x_{t}^{1}, \dots, x_{t}^{N}\} \in \overline{\mathbb{X}}^{N}$ denote the response sequence at time step $t \in \{0, \dots, T\}$, 
where $x_{t}^{r}$ is the token at position $r$. The process evolves from a fully masked state $\mathbf{x}_{T} = \{\operatorname{M}\}^{N}$ to a fully unmasked state $\mathbf{x}_{0} \in \mathbb{X}^{N}$ 
drawn from a training dataset $\mathcal{D}$. We use $\mathcal{I}=\{1,\dots,N\}$ for the token index set, reserve $t,t'$ for decoding timesteps, and reserve $r,s,k$ for token indices. Let $U_t \subseteq \mathcal{I}$ and $M_t=\mathcal{I}\setminus U_t$ denote the unmasked and masked index sets at timestep $t$, respectively. 
For convenience, we use $\mathbf{x}_{0}^{U_t}$ to represent the visible context formed by the unmasked tokens in $U_t$. 
Standard training of DLMs typically uses the NELBO objective with uniform masking, 
where the unmasked set $U$ is sampled uniformly to serve as the visible context, and the masked tokens in $M=\mathcal{I}\setminus U$ are reconstructed in a single step:
\begin{equation}
  \label{reconstruction loss}
  \mathcal{L}_{\mathrm{NELBO}} = \mathbb{E}_{\mathbf{x}_0 \sim \mathcal{D}} \, \mathbb{E}_{U \sim q_{\mathrm{unif}}(U)} \left[ \frac{1}{|M|} \sum_{r \in M} -\log p_\theta(x_{0}^{r} \mid \mathbf{x}_{0}^{U}, \mathbf{s}) \right]
\end{equation}
\begin{proposition}[Uniform Inductive Bias of NELBO]
\label{prop:uniform_bias}
Under the NELBO objective, the model is optimized to reconstruct all masked tokens simultaneously given a uniformly sampled context. This optimization inherently instills a uniform inductive bias that treats all tokens equally regardless of their inherent prediction difficulty, thereby failing to capture the sequential dependencies and certainty gradients required for coherent generation.
\end{proposition} 

Proposition \ref{prop:uniform_bias} mainly implies the uniform inductive bias of the existing NELBO objective. 

In contrast, actual inference decoding adopted by existing DLM models \citep{li2025convergencetheorydiffusionlanguage,ye2025dream,labs2025mercuryultrafastlanguagemodels} is a confidence- or entropy-guided multi-step process. Specifically, starting from an empty unmasked set $U_T = \emptyset$, each reverse decoding step $t$ selects a subset of tokens $\mathcal{J}_t$ of size $b_t$ from the masked set $M_t$ to unmask by minimizing the predictive entropy:
\begin{equation}
\label{eq:decoding}
    \mathcal{J}_{t} = \underset{\mathcal{S} \subset M_{t}, |\mathcal{S}|=b_t}{\arg\min} \sum_{r \in \mathcal{S}} \mathcal{H}_\theta(x_{0}^{r} \mid \mathbf{x}_{0}^{U_t}, \mathbf{s}), 
    \quad U_{t-1}=U_t\cup \mathcal{J}_t .
\end{equation}
where $\mathcal{H}_\theta$ denotes the predictive entropy, which can also be served as a smooth and robust approximation to confidence-based metrics (e.g., margin confidence), particularly given that LLM distributions tend to be highly overconfident \citep{sun2025largelanguagemodelsoverconfident}. By iteratively selecting tokens based on this certainty metric, the decoding mechanism fundamentally shapes the generation trajectory, which can be formally summarized as follows.

\vspace{-0.4\baselineskip}
\begin{proposition}[Easy-to-Hard Inductive Bias of Inference]
\label{prop:easy_to_hard_bias}
The entropy-guided decoding schedule inherently instills a strong easy-to-hard inductive bias, forming a stark contrast to the uniform inductive bias of the NELBO objective.
\end{proposition}
\vspace{-0.7\baselineskip}
\begin{proof}
Let the prediction difficulty of a token $r$ be quantified by its predictive entropy $\mathcal{H}_\theta(x_{0}^{r} \mid \mathbf{x}_{0}^{U_t}, \mathbf{s})$. A lower entropy indicates higher predictive certainty, corresponding to an ``easier'' token. At any step $t$, the decoding objective (Eq.~\ref{eq:decoding}) explicitly minimizes this entropy, thereby selecting the easiest remaining tokens in the masked set $M_t$ to transition into the next unmasked set $U_{t-1}$. Consequently, the unmasking sequence monotonically progresses from the most certain (easiest) tokens to the least certain (hardest) tokens, establishing an implicit easy-to-hard inductive bias.
\end{proof}
\vspace{-0.7\baselineskip}
To formalize this discrepancy between the uniform training bias and the easy-to-hard inference bias, we explicitly define the target distribution generated during inference.

\begin{definition}[Inference Unmasked Distribution]
\label{def:infer_dist}
For a base DLM $\theta_{\mathrm{base}}$, let a self-distilled trajectory be $\tau=(U_T,U_{T-1},\ldots,U_0)$, where $U_T=\emptyset$, $U_0=\mathcal{I}$, and $U_{t-1}\supseteq U_t$. We denote the collection of such trajectories by $\mathcal{T}_{\text{gold}}$. The distribution $q_{\mathrm{infer}}(U \mid \mathbf{x}_0, \theta_{\mathrm{base}})$ denotes the marginal distribution over unmasked states $U$ visited by these entropy-guided trajectories. For a token $r$, $\operatorname{step}_{\tau}(r)$ denotes the reverse timestep immediately before $r$ is decoded; if several tokens are decoded in the same model step, ties are resolved by their entropy ranks within that step. Thus, $q_{\mathrm{infer}}$ concentrates probability mass on states reachable through high-certainty, easy-to-hard unmasking decisions.
\end{definition}

Consequently, a severe distribution discrepancy exists between training and inference: $D_{\mathrm{KL}}(q_{\mathrm{infer}}(U \mid \mathbf{x}_0, \theta_{\mathrm{base}}) \,\|\, q_{\mathrm{unif}}(U)) \gg 0$. 
This large divergence arises because the training distribution $q_{\mathrm{unif}}$ assigns equal probability to all possible unmasking combinations, 
whereas the inference distribution $q_{\mathrm{infer}}$ is highly peaked, concentrating its mass on a narrow set of easy-to-hard trajectories. 
This theoretical gap derives the following observation:
\begin{tcolorbox}[colback=gray!10, colframe=black, arc=4pt, boxrule=0.5pt, left=4pt, right=4pt, top=4pt, bottom=4pt]
\textit{Although directly utilizing inference trajectories for SFT provides the model with exposure to intermediate reasoning states, this fundamental distribution discrepancy prevents the model from fully exploiting the structured information embedded within the trajectories, leading to sub-optimal performance gains.}
\end{tcolorbox}

\subsection{Trajectory-Aligned Optimization via Boltzmann Modeling}
\label{sec:tao}

To resolve this distribution discrepancy and fully unlock the potential of self-distilled trajectories, our core idea is to explicitly model the inference-time decoding behavior $q_{\mathrm{infer}}(U \mid \mathbf{x}_0, \theta_{\mathrm{base}})$ and integrate it into the SFT process. By doing so, we can reduce the discrepancy between training and inference, aligning the model's optimization landscape with its actual generation dynamics. 
\begin{samepage}
\begin{proposition}[Boltzmann Unmasked Distribution]
\label{prop:boltzmann}
Under the easy-to-hard decoding assumption, the idealized target inference distribution $q_{\mathrm{infer}}^{\star}$ can be modeled as a Boltzmann distribution over the ideal predictive entropy $\mathcal{H}_{\mathrm{ideal}}(x_{0}^{r})$ of each token with an inverse temperature parameter $\beta$:
\begin{equation}
  q_{\mathrm{infer}}^{\star}(U \mid \mathbf{x}_0) = \frac{1}{Z} \exp\left( -\beta \sum_{r \in U} \mathcal{H}_{\mathrm{ideal}}(x_{0}^{r}) \right)
  \label{eq:boltzmann}
\end{equation}
where $Z = \sum_{U' \subseteq \mathcal{I}} \exp\left( -\beta \sum_{r \in U'} \mathcal{H}_{\mathrm{ideal}}(x_{0}^{r}) \right)$ is the partition function.
\end{proposition}
\begin{proof}
Intuitively, we can view a token's predictive entropy as its ``cost''. The easy-to-hard decoding process naturally gravitates towards paths with the lowest total cost. Therefore, an unmasked set $U$ consisting of easy tokens (low total entropy) is highly likely to be formed, whereas a set containing hard tokens is exponentially less likely. This exponential preference for low-energy states is exactly described by the Boltzmann distribution, where the probability of a state is proportional to $\exp(-\beta \cdot \text{cost})$. Substituting the total entropy $\sum_{r \in U} \mathcal{H}_{\mathrm{ideal}}(x_{0}^{r})$ as the cost yields Eq.~\ref{eq:boltzmann}.
\end{proof}
\end{samepage}
Here, $\mathcal{H}_{\mathrm{ideal}}(x_{0}^{r})$ denotes the ideal predictive entropy under the true generative manifold (i.e., the inherent structural uncertainty of predicting $x_{0}^{r}$ given its optimally decoded context). 
In practice, we approximate $\mathcal{H}_{\mathrm{ideal}}(x_{0}^{r})$ using the base model's predicted entropy $\mathcal{H}_{\theta_{\text{base}}}(x_{0}^{r} \mid \mathbf{x}_{0}^{U_t}, \mathbf{s})$ on its own self-distilled trajectories. This derived formulation assigns high probability mass to unmasked sets dominated by low-entropy tokens, and near-zero probability mass to unmasked sets consisting entirely of high-entropy tokens, effectively capturing the easy-to-hard inductive bias.

\textbf{The Ideal Objective.} Therefore, to close the gap between the uniform training distribution and the easy-to-hard inference distribution, we propose to minimize the Kullback-Leibler (KL) divergence between the model's predicted unmasked distribution $p_\theta$ and the target inference Boltzmann distribution $q_{\mathrm{infer}}^{\star}$, and integrate this alignment term into the standard NELBO-based SFT objective:
\begin{equation}
  \mathcal{L}_{\mathrm{TD}}^{\mathrm{ideal}} = \mathcal{L}_{\mathrm{NELBO}} + \lambda \min_\theta D_{\mathrm{KL}}(q_{\mathrm{infer}}^{\star} \,\|\, p_\theta)
\end{equation}
By optimizing this objective, we explicitly force the model's predictive certainty to match the actual easy-to-hard decoding trajectory. To derive a tractable form for this KL divergence, we first parameterize the model's predicted probability of an unmasked set $U$ using the energy score induced by its predictive entropies:
\begin{equation}
  p_\theta(U \mid \mathbf{x}_0,\mathbf{s}) = \frac{1}{Z_\theta} \exp\left( -\beta \sum_{r \in U} h_\theta(r; \tau) \right),
  \quad h_\theta(r; \tau) \triangleq \mathcal{H}_\theta(x_{0}^{r} \mid \mathbf{x}_{0}^{U_{\operatorname{step}_{\tau}(r)}}, \mathbf{s}),
\end{equation}
where $\operatorname{step}_{\tau}(r)$ denotes the timestep immediately before token $r$ is decoded along trajectory $\tau$. Based on this parameterization, the KL divergence term expands as follows, 
\begin{align}
  \min_\theta D_{\mathrm{KL}}(q_{\mathrm{infer}}^{\star} \,\|\, p_\theta) &= \min_\theta \mathbb{E}_{U \sim q_{\mathrm{infer}}^{\star}} \left[ \log q_{\mathrm{infer}}^{\star}(U \mid \mathbf{x}_0) - \log p_\theta(U \mid \mathbf{x}_0,\mathbf{s}) \right] \nonumber \\
  &= \max_\theta \mathbb{E}_{U \sim q_{\mathrm{infer}}^{\star}} [ \log p_\theta(U \mid \mathbf{x}_0,\mathbf{s}) ]  \nonumber \\
  &= \max_\theta \mathbb{E}_{U \sim q_{\mathrm{infer}}^{\star}} \left[ -\beta \sum_{r \in U} h_\theta(r;\tau) - \log Z_\theta \right]
  \label{eq:kl_expansion}
\end{align}
Directly optimizing this expectation is intractable because we cannot easily sample from the global dynamic distribution $q_{\mathrm{infer}}^{\star}$, and computing the partition function $Z_\theta$ requires a combinatorial sum over all possible unmasked states.

\subsection{The Tractable Surrogate: Pairwise Ranking}
\label{sec:pairwise_ranking}

To bypass the intractable partition function and the need for global sampling, we shift our perspective from explicit likelihood maximization to implicit distribution matching via ranking.

\begin{theorem}[Energy-Based Ranking \citep{lecun2006tutorial}]
\label{thm:energy_ranking}
Maximum likelihood estimation, which is equivalent to minimizing the KL divergence, can be effectively transformed into an energy-based ranking problem. By enforcing correct relative rankings between states, one can shape the energy landscape and implicitly match the target distribution without the intractable partition function.
\end{theorem}

Based on Theorem \ref{thm:energy_ranking}, we derive the following lemma specifically for our objective of aligning unmasked distributions.

\begin{lemma}[Ranking Unmasked States]
\label{lemma:unmasked_ranking}
Let $S_\theta(U;\tau) = -\beta \sum_{r \in U} h_\theta(r;\tau)$ denote the unnormalized energy score of an unmasked set $U$ along trajectory $\tau$. To align the model distribution $p_\theta$ with the target inference distribution $q_{\mathrm{infer}}^{\star}$, for any pair of unmasked sets where $q_{\mathrm{infer}}^{\star}(U_b) > q_{\mathrm{infer}}^{\star}(U_a)$, it is sufficient to enforce the corresponding model ranking $S_\theta(U_b;\tau) > S_\theta(U_a;\tau)$.
\end{lemma}

Lemma \ref{lemma:unmasked_ranking} implies that optimizing our objective relies on identifying pairs of unmasked states with a known relative preference under the target inference distribution. Therefore, \textbf{self-distilled trajectories} provide the exact structured data required to construct these pairs and learn this ranking objective. Because these trajectories record the actual step-by-step unmasking process of the base model, they naturally reflect the relative predictive certainty between different tokens. 

Specifically, consider a self-distilled trajectory $\tau \sim \mathcal{T}_{\text{gold}}$ and two token indices $r$ and $s$ decoded at different positions along this trajectory. If $r$ is decoded before $s$ (\texttt{i.e.,} $\operatorname{step}_{\tau}(r) > \operatorname{step}_{\tau}(s)$ under reverse-time decoding), then $r$ is treated as an easier token with lower ideal predictive entropy. Let $U_{\text{base}}$ be a shared subset of unmasked tokens, and define two unmasked sets differing by only one token: $U_r = U_{\text{base}} \cup \{r\}$ and $U_s = U_{\text{base}} \cup \{s\}$. Under the target Boltzmann distribution (Eq.~\ref{eq:boltzmann}), the unmasked set containing the easier token has a higher probability, so $q_{\mathrm{infer}}^{\star}(U_r) > q_{\mathrm{infer}}^{\star}(U_s)$. Enforcing corresponding model ranking $S_\theta(U_r;\tau) > S_\theta(U_s;\tau)$ allows us to bypass the intractable $Z_\theta$. Since $U_r$ and $U_s$ share the common subset $U_{\text{base}}$, their scores simplify:
\begin{equation}
  S_\theta(U_r;\tau) > S_\theta(U_s;\tau) \implies h_\theta(r;\tau) < h_\theta(s;\tau).
\end{equation}
To enforce this inequality, we adopt a pairwise hinge loss that penalizes the model whenever the predicted entropy of an easier token $r$ is not sufficiently lower than that of a harder token $s$. For a local trajectory segment $(U_t,U_{t'})$ with $t' < t$, define the newly unmasked token set as $\Delta_{t \to t'} = U_{t'} \setminus U_t$ and the local ordered pair set as
\[
  \mathcal{P}_{t,t'}=\{(r,s): r,s\in \Delta_{t \to t'},\ \operatorname{step}_{\tau}(r)>\operatorname{step}_{\tau}(s)\}.
\]
When $|\Delta_{t \to t'}|=W$, we have $|\mathcal{P}_{t,t'}|=W(W-1)/2$. The ranking loss over this segment is
\begin{equation}
  \mathcal{L}_{\mathrm{rank}}(t,t') = \frac{1}{|\mathcal{P}_{t,t'}|} \sum_{(r, s) \in \mathcal{P}_{t,t'}} \max\Big(0, \, h_\theta(r;\tau) - h_\theta(s;\tau) + \gamma \Big),
\end{equation}
where $\gamma>0$ is the ranking margin. Since the surrogate only depends on relative ordering, the inverse-temperature scale $\beta$ is absorbed into the margin $\gamma$ and the ranking weight $\lambda$ in practice. As guaranteed by Lemma \ref{lemma:unmasked_ranking}, this ranking mechanism serves as an effective surrogate for the intractable KL divergence (Eq.~\ref{eq:kl_expansion}). It avoids global sampling and partition function computation by implicitly matching the target Boltzmann distribution through entropy landscape shaping. Integrating this ranking regularization with the reconstruction loss, our final objective is
\begin{equation}
  \label{eq:final_objective}
  \begin{aligned}
   \mathcal{L}_{\mathrm{TABOM}}
   =
   \mathbb{E}_{\tau \sim \mathcal{T}_{\text{gold}}} \,
   \mathbb{E}_{(t,t') \sim \tau}
   \Bigg[
   &\frac{1}{W} \sum_{k \in \Delta_{t \to t'}}
   -\log p_\theta(x_{0}^{k} \mid \mathbf{x}_{0}^{U_t}, \mathbf{s}) \\
   &+ \frac{2\lambda}{W(W-1)}
   \sum_{(r, s) \in \mathcal{P}_{t,t'}}
   \max\Big(0, \, h_\theta(r;\tau) - h_\theta(s;\tau) + \gamma \Big)
   \Bigg].
  \end{aligned}
\end{equation}
where $(t,t')\sim\tau$ means sampling a valid local segment from trajectory $\tau$ such that $t'<t$ and $|\Delta_{t\to t'}|=W$. The first term denotes the trajectory-aware reconstruction loss, which replaces the uniform masking ($U \sim q_{\mathrm{unif}}$) with a trajectory-aware masking strategy as widely adopted by previous works \citep{ma2025dinferefficientinferenceframework} for fully leveraging the self-distilled trajectories. Specifically, the unmasked tokens at an earlier state $U_t$ are used as context to predict the newly unmasked tokens $\Delta_{t \to t'} = U_{t'} \setminus U_t$ in the subsequent state $U_{t'}$ ($t' < t$).
Note that we do not construct the pair set globally across all unmasked tokens. Comparing tokens decoded at vastly different stages (e.g., the first vs. the last step) is less informative, as their prediction difficulties are inherently disparate (see analysis in Appendix \ref{Sensitivity W}). Instead, we apply the ranking loss strictly over a local decoding window containing $W=|\Delta_{t \to t'}|$ newly unmasked tokens.
To avoid the instability of the ranking loss that occurs when the target state is either too far from $U_t$ or too close to $U_t$, we set a fixed window size $W$ (e.g., $W=32$). Consequently, during training, we only need to randomly sample the initial timestep $t$ and its context $U_t$; the target state $U_{t'}$ is then determined by $W$ newly unmasked tokens.

\section{Experiments}
\label{sec:Experiments}
\textbf{Models.} We employ two state-of-the-art diffusion language models as our primary experimental testbeds: Dream-7B-Instruct \citep{ye2025dream} and LLaDA-8B-Instruct \citep{nie2025large}. These models serve as the baselines upon which we apply our post-training recipes.

\textbf{Self-distilled and Offline Ground-truth Dataset.} We focus on two core domains: mathematical reasoning and code generation. For mathematical reasoning, we utilize 12K queries from MixChain-Z-PRM12K \citep{horseee_mixchain}. For code generation, we sample 18K queries from Ling-Coder-SFT \citep{inclusionAI_ling_coder}. For each query, the offline ground-truth data consists of the standard problem-answer pairs originally provided in the datasets. To construct the self-distilled data, we generate trajectories using the base models' default decoding strategies. Specifically, we employ entropy-based decoding for Dream and confidence-based decoding for LLaDA, limiting the generation sequence length to 256 tokens.

\textbf{Training.} By following previous post-training-related counterparts~\citep{zhao2025d1}, we train a separate model for each domain dataset. We employ LoRA \citep{hu2021loralowrankadaptationlarge} for parameter-efficient fine-tuning, with rank $r=16$, LoRA scaling $\alpha=16$, and target modules \texttt{q\_proj} and \texttt{v\_proj}. The training is conducted across 8 GPUs with a per-device mini-batch size of 4. We optimize the models with AdamW using a peak learning rate of 2e-5, together with a cosine learning rate decay schedule and a warm-up period of 50 steps. All models are trained for 5 epochs and we report the best results for all baselines. For TABOM, the window size $W$ is fixed to 32 across all tasks. The margin $\gamma$ and the ranking loss weight $\lambda$ in Eq.~\ref{eq:final_objective} are tuned via a grid search over $\{0.1, 0.2, 0.3\}$ and $\{1, 2\}$, respectively, for each task. Specific hyperparameters, training details, and evaluation setups are detailed in Appendix.

\textbf{Evaluation Tasks.} We conduct experiments on five main tasks grouped into three categories: (1) \textit{Mathematical reasoning}: GSM8K \citep{cobbe2021gsm8k} and MATH500 \citep{hendrycksmath2021}; (2) \textit{Code generation}: HumanEval \citep{chen2021codex} and MBPP \citep{austin2021structured}; (3) \textit{Instruction following}: IFEval \citep{zhou2023instruction}, which evaluates the model's ability to follow verifiable instructions. 
Note that for each trained model, we evaluate both its in-domain and out-of-distribution (OOD) performance. For instance, for a model trained on the code generation dataset, HumanEval and MBPP serve as in-domain evaluations, while the remaining three tasks are considered OOD. 

\textbf{Baselines.} We evaluate original DLMs under its official default inference setting (\textbf{No-SFT}), and further compare our approach with four post-training baselines. The LoRA configuration and the training data for all post-training baselines are the same as our method (i.e., using the self-distilled data), except for SFT-GT which utilizes the offline ground-truth data: (1) \textbf{SFT-GT}: Standard supervised fine-tuning using offline ground-truth data, where the DLM is trained to reconstruct randomly masked tokens. (2) \textbf{SFT-SD}: Identical to SFT-GT, but trained on the self-distilled data. (3) \textbf{dInfer}~\citep{ma2025dinferefficientinferenceframework}: Instead of
learning the standard single-step transition, dInfer conducts compressed transition learning by randomly sampling two timesteps from self-distilled trajectories. (4) \textbf{T3D}~\citep{zhang2026t3d}: Training on the self-distilled data with a randomly sampling timestep, which applies direct discriminative optimization and path-consistency weighting to the objective.

\subsection{Main Results}
\label{sec:Experiments1}
Table \ref{tab:coding_results} and Table \ref{tab:math_results} summarize the performance of our proposed TABOM framework compared to various baselines on the code generation and mathematical reasoning datasets, respectively. We observe several key findings: \textbf{(1) SFT-GT suffers from severe catastrophic forgetting.} While it improves in-domain performance (e.g., HumanEval on Dream code increases by 8.89\%), it drastically degrades OOD capabilities (e.g., GSM8K drops by 29.08\%). This highlights the optimization barrier of directly injecting offline ground-truth data into the diffusion model's manifold.
\textbf{(2) SFT-SD prevents forgetting but yields marginal gains.} By utilizing self-distilled trajectories, SFT-SD effectively preserves the model's inherent predictive manifold, maintaining OOD performance. However, due to the training-inference misalignment, its in-domain improvements are minimal.
\textbf{(3) Trajectory-based baselines (dInfer, T3D) offer limited improvements.} While they attempt to leverage decoding trajectories, they still struggle to fully unlock the potential of the self-distilled data, resulting in sub-optimal average gains.
Finally, \textbf{TABOM achieves the best of both worlds.} By explicitly modeling the inference-time decoding behavior and aligning the predictive certainty via pairwise ranking, TABOM consistently outperforms all baselines across both base models and domains. It not only achieves the highest in-domain performance (e.g., +5.15\% average gain on Dream code) but also expands the knowledge boundaries, yielding positive OOD gains without catastrophic forgetting.

\begin{table}[t]
  \centering
  \caption{Performance comparison of models fine-tuned on the \textbf{Code Generation} dataset. We report the absolute scores and their relative gains/losses compared to the No-SFT baseline. The highest and second-highest values in each column  are highlighted in \textbf{bold} and \underline{underlined}, respectively.}
  \label{tab:coding_results}
  \resizebox{\linewidth}{!}{
  \begin{tabular}{l|ccc|cccc}
  \toprule
  \multirow{2}{*}{\textbf{Method}} & \multicolumn{3}{c|}{\textbf{In-Domain}} & \multicolumn{4}{c}{\textbf{Out-Of-Distribution (OOD)}} \\
  \cmidrule(lr){2-4} \cmidrule(lr){5-8}
  & \textbf{HumanEval} & \textbf{MBPP} & \textbf{Avg.} & \textbf{GSM8K} & \textbf{MATH500} & \textbf{IFEval} & \textbf{Avg.} \\
  \midrule
  \multicolumn{8}{c}{\textit{Base Model: Dream-7B-Instruct}} \\
  \midrule
  \rowcolor{gray!15} No-SFT & 52.66 & 58.00 & 55.33 & 81.41 & 39.80 & 56.56 & 59.26 \\
  SFT-GT & \textbf{61.55}$_{\color{green!60!black}{\mathbf{+8.89}}}$ & 58.00$_{\color{gray}{\mathbf{+0.00}}}$ & \underline{59.78}$_{\color{green!60!black}{\mathbf{+4.45}}}$ & 52.33$_{\color{red}{\mathbf{-29.08}}}$ & 32.40$_{\color{red}{\mathbf{-7.40}}}$ & 46.21$_{\color{red}{\mathbf{-10.35}}}$ & 43.65$_{\color{red}{\mathbf{-15.61}}}$ \\
  SFT-SD & 53.66$_{\color{green!60!black}{\mathbf{+1.00}}}$ & \underline{59.20}$_{\color{green!60!black}{\mathbf{+1.20}}}$ & 56.43$_{\color{green!60!black}{\mathbf{+1.10}}}$ & 81.81$_{\color{green!60!black}{\mathbf{+0.40}}}$ & \underline{41.60}$_{\color{green!60!black}{\mathbf{+1.80}}}$ & 57.10$_{\color{green!60!black}{\mathbf{+0.54}}}$ & \textbf{60.17}$_{\color{green!60!black}{\mathbf{+0.91}}}$ \\
  dInfer & 57.31$_{\color{green!60!black}{\mathbf{+4.65}}}$ & 58.20$_{\color{green!60!black}{\mathbf{+0.20}}}$ & 57.76$_{\color{green!60!black}{\mathbf{+2.43}}}$ & \textbf{81.88}$_{\color{green!60!black}{\mathbf{+0.47}}}$ & 39.80$_{\color{gray}{\mathbf{+0.00}}}$ & \textbf{57.30}$_{\color{green!60!black}{\mathbf{+0.74}}}$ & 59.66$_{\color{green!60!black}{\mathbf{+0.40}}}$ \\
  T3D    & 55.48$_{\color{green!60!black}{\mathbf{+2.82}}}$ & 58.70$_{\color{green!60!black}{\mathbf{+0.70}}}$ & 57.09$_{\color{green!60!black}{\mathbf{+1.76}}}$ & \underline{81.84}$_{\color{green!60!black}{\mathbf{+0.43}}}$ & 40.70$_{\color{green!60!black}{\mathbf{+0.90}}}$ & \underline{57.20}$_{\color{green!60!black}{\mathbf{+0.64}}}$ & \underline{59.91}$_{\color{green!60!black}{\mathbf{+0.65}}}$ \\
  \rowcolor{blue!10} TABOM (Ours) & \underline{60.36}$_{\color{green!60!black}{\mathbf{+7.70}}}$ & \textbf{60.60}$_{\color{green!60!black}{\mathbf{+2.60}}}$ & \textbf{60.48}$_{\color{green!60!black}{\mathbf{+5.15}}}$ & 81.73$_{\color{green!60!black}{\mathbf{+0.32}}}$ & \textbf{42.40}$_{\color{green!60!black}{\mathbf{+2.60}}}$ & 55.45$_{\color{red}{\mathbf{-1.11}}}$ & 59.86$_{\color{green!60!black}{\mathbf{+0.60}}}$ \\
  \midrule
  \multicolumn{8}{c}{\textit{Base Model: LLaDA-8B-Instruct}} \\
  \midrule
  \rowcolor{gray!15} No-SFT & 36.01 & 39.20 & 37.61 & 76.12 & 36.20 & 33.08 & 48.47 \\
  SFT-GT & \underline{42.01}$_{\color{green!60!black}{\mathbf{+6.00}}}$ & 32.80$_{\color{red}{\mathbf{-6.40}}}$ & 37.41$_{\color{red}{\mathbf{-0.20}}}$ & 70.73$_{\color{red}{\mathbf{-5.39}}}$ & 35.40$_{\color{red}{\mathbf{-0.80}}}$ & \textbf{34.93}$_{\color{green!60!black}{\mathbf{+1.85}}}$ & 47.02$_{\color{red}{\mathbf{-1.45}}}$ \\
  SFT-SD & 39.63$_{\color{green!60!black}{\mathbf{+3.62}}}$ & \underline{38.80}$_{\color{red}{\mathbf{-0.40}}}$ & 39.22$_{\color{green!60!black}{\mathbf{+1.61}}}$ & 76.95$_{\color{green!60!black}{\mathbf{+0.83}}}$ & 35.80$_{\color{red}{\mathbf{-0.40}}}$ & 32.90$_{\color{red}{\mathbf{-0.18}}}$ & 48.55$_{\color{green!60!black}{\mathbf{+0.08}}}$ \\
  dInfer & 41.46$_{\color{green!60!black}{\mathbf{+5.45}}}$ & 38.60$_{\color{red}{\mathbf{-0.60}}}$ & \underline{40.03}$_{\color{green!60!black}{\mathbf{+2.42}}}$ & \textbf{77.33}$_{\color{green!60!black}{\mathbf{+1.21}}}$ & \underline{36.60}$_{\color{green!60!black}{\mathbf{+0.40}}}$ & 33.82$_{\color{green!60!black}{\mathbf{+0.74}}}$ & \underline{49.25}$_{\color{green!60!black}{\mathbf{+0.78}}}$ \\
  T3D    & 40.54$_{\color{green!60!black}{\mathbf{+4.53}}}$ & 38.70$_{\color{red}{\mathbf{-0.50}}}$ & 39.62$_{\color{green!60!black}{\mathbf{+2.01}}}$ & \underline{77.14}$_{\color{green!60!black}{\mathbf{+1.02}}}$ & 36.20$_{\color{gray}{\mathbf{+0.00}}}$ & 33.36$_{\color{green!60!black}{\mathbf{+0.28}}}$ & 48.90$_{\color{green!60!black}{\mathbf{+0.43}}}$ \\
  \rowcolor{blue!10} TABOM (Ours) & \textbf{42.68}$_{\color{green!60!black}{\mathbf{+6.67}}}$ & \textbf{40.00}$_{\color{green!60!black}{\mathbf{+0.80}}}$ & \textbf{41.34}$_{\color{green!60!black}{\mathbf{+3.73}}}$ & \textbf{77.33}$_{\color{green!60!black}{\mathbf{+1.21}}}$ & \textbf{38.20}$_{\color{green!60!black}{\mathbf{+2.00}}}$ & \underline{34.19}$_{\color{green!60!black}{\mathbf{+1.11}}}$ & \textbf{49.91}$_{\color{green!60!black}{\mathbf{+1.44}}}$ \\
  \bottomrule
  \end{tabular}
  }
  \vspace{-0.5cm}
\end{table}

\begin{table}[t]
  \centering
  \caption{Performance comparison of models fine-tuned on the \textbf{Mathematical Reasoning} dataset. We report the absolute scores and their relative gains/losses compared to the No-SFT baseline. The highest and second-highest values in each column are highlighted in \textbf{bold} and \underline{underlined}, respectively.}
  \label{tab:math_results}
  \resizebox{\linewidth}{!}{
  \begin{tabular}{l|ccc|cccc}
  \toprule
  \multirow{2}{*}{\textbf{Method}} & \multicolumn{3}{c|}{\textbf{In-Domain}} & \multicolumn{4}{c}{\textbf{Out-Of-Distribution (OOD)}} \\
  \cmidrule(lr){2-4} \cmidrule(lr){5-8}
  & \textbf{GSM8K} & \textbf{MATH500} & \textbf{Avg.} & \textbf{HumanEval} & \textbf{MBPP} & \textbf{IFEval} & \textbf{Avg.} \\
  \midrule
  \multicolumn{8}{c}{\textit{Base Model: Dream-7B-Instruct}} \\
  \midrule
  \rowcolor{gray!15} No-SFT & 81.41 & 39.80 & 60.61 & 52.66 & 58.00 & 56.56 & 55.74 \\
  SFT-GT & 80.12$_{\color{red}{\mathbf{-1.29}}}$ & 37.40$_{\color{red}{\mathbf{-2.40}}}$ & 58.76$_{\color{red}{\mathbf{-1.85}}}$ & 46.34$_{\color{red}{\mathbf{-6.32}}}$ & 58.00$_{\color{gray}{\mathbf{+0.00}}}$ & 53.23$_{\color{red}{\mathbf{-3.33}}}$ & 52.52$_{\color{red}{\mathbf{-3.22}}}$ \\
  SFT-SD & 81.95$_{\color{green!60!black}{\mathbf{+0.54}}}$ & 39.80$_{\color{gray}{\mathbf{+0.00}}}$ & 60.88$_{\color{green!60!black}{\mathbf{+0.27}}}$ & \underline{57.92}$_{\color{green!60!black}{\mathbf{+5.26}}}$ & 58.60$_{\color{green!60!black}{\mathbf{+0.60}}}$ & \underline{56.01}$_{\color{red}{\mathbf{-0.55}}}$ & \underline{57.51}$_{\color{green!60!black}{\mathbf{+1.77}}}$ \\
  dInfer & \underline{82.33}$_{\color{green!60!black}{\mathbf{+0.92}}}$ & \textbf{41.60}$_{\color{green!60!black}{\mathbf{+1.80}}}$ & \underline{61.97}$_{\color{green!60!black}{\mathbf{+1.36}}}$ & 56.11$_{\color{green!60!black}{\mathbf{+3.45}}}$ & \underline{58.80}$_{\color{green!60!black}{\mathbf{+0.80}}}$ & 55.82$_{\color{red}{\mathbf{-0.74}}}$ & 56.91$_{\color{green!60!black}{\mathbf{+1.17}}}$ \\
  T3D    & 82.14$_{\color{green!60!black}{\mathbf{+0.73}}}$ & 40.70$_{\color{green!60!black}{\mathbf{+0.90}}}$ & 61.42$_{\color{green!60!black}{\mathbf{+0.81}}}$ & 57.01$_{\color{green!60!black}{\mathbf{+4.35}}}$ & 58.70$_{\color{green!60!black}{\mathbf{+0.70}}}$ & 55.91$_{\color{red}{\mathbf{-0.65}}}$ & 57.21$_{\color{green!60!black}{\mathbf{+1.47}}}$ \\
  \rowcolor{blue!10} TABOM (Ours) & \textbf{84.31}$_{\color{green!60!black}{\mathbf{+2.90}}}$ & \underline{41.10}$_{\color{green!60!black}{\mathbf{+1.30}}}$ & \textbf{62.71}$_{\color{green!60!black}{\mathbf{+2.10}}}$ & \textbf{58.54}$_{\color{green!60!black}{\mathbf{+5.88}}}$ & \textbf{59.20}$_{\color{green!60!black}{\mathbf{+1.20}}}$ & \textbf{56.19}$_{\color{red}{\mathbf{-0.37}}}$ & \textbf{57.98}$_{\color{green!60!black}{\mathbf{+2.24}}}$ \\
  \midrule
  \multicolumn{8}{c}{\textit{Base Model: LLaDA-8B-Instruct}} \\
  \midrule
  \rowcolor{gray!15} No-SFT & 76.12 & 36.20 & 56.16 & 36.01 & 39.20 & 33.08 & 36.10 \\
  SFT-GT & 74.29$_{\color{red}{\mathbf{-1.83}}}$ & 35.50$_{\color{red}{\mathbf{-0.70}}}$ & 54.90$_{\color{red}{\mathbf{-1.26}}}$ & 31.09$_{\color{red}{\mathbf{-4.92}}}$ & \textbf{40.60}$_{\color{green!60!black}{\mathbf{+1.40}}}$ & 26.98$_{\color{red}{\mathbf{-6.10}}}$ & 32.89$_{\color{red}{\mathbf{-3.21}}}$ \\
  SFT-SD & 75.96$_{\color{red}{\mathbf{-0.16}}}$ & 35.70$_{\color{red}{\mathbf{-0.50}}}$ & 55.83$_{\color{red}{\mathbf{-0.33}}}$ & 36.58$_{\color{green!60!black}{\mathbf{+0.57}}}$ & 39.80$_{\color{green!60!black}{\mathbf{+0.60}}}$ & \textbf{34.38}$_{\color{green!60!black}{\mathbf{+1.30}}}$ & 36.92$_{\color{green!60!black}{\mathbf{+0.82}}}$ \\
  dInfer & \underline{76.72}$_{\color{green!60!black}{\mathbf{+0.60}}}$ & \underline{36.50}$_{\color{green!60!black}{\mathbf{+0.30}}}$ & \underline{56.61}$_{\color{green!60!black}{\mathbf{+0.45}}}$ & \underline{38.90}$_{\color{green!60!black}{\mathbf{+2.89}}}$ & 39.40$_{\color{green!60!black}{\mathbf{+0.20}}}$ & \textbf{34.38}$_{\color{green!60!black}{\mathbf{+1.30}}}$ & \underline{37.56}$_{\color{green!60!black}{\mathbf{+1.46}}}$ \\
  T3D    & 76.34$_{\color{green!60!black}{\mathbf{+0.22}}}$ & 36.10$_{\color{red}{\mathbf{-0.10}}}$ & 56.22$_{\color{green!60!black}{\mathbf{+0.06}}}$ & 37.74$_{\color{green!60!black}{\mathbf{+1.73}}}$ & 39.60$_{\color{green!60!black}{\mathbf{+0.40}}}$ & \textbf{34.38}$_{\color{green!60!black}{\mathbf{+1.30}}}$ & 37.24$_{\color{green!60!black}{\mathbf{+1.14}}}$ \\
  \rowcolor{blue!10} TABOM (Ours) & \textbf{78.62}$_{\color{green!60!black}{\mathbf{+2.50}}}$ & \textbf{36.80}$_{\color{green!60!black}{\mathbf{+0.60}}}$ & \textbf{57.71}$_{\color{green!60!black}{\mathbf{+1.55}}}$ & \textbf{40.30}$_{\color{green!60!black}{\mathbf{+4.29}}}$ & \underline{40.10}$_{\color{green!60!black}{\mathbf{+0.90}}}$ & \underline{32.98}$_{\color{red}{\mathbf{-0.10}}}$ & \textbf{37.79}$_{\color{green!60!black}{\mathbf{+1.69}}}$ \\
  \bottomrule
  \end{tabular}
  }
  \vspace{-0.4cm}
\end{table}

\textbf{Robustness in Parallel Decoding.} Although our primary focus is not on accelerating inference, we are also interested in the performance of these models under highly parallel generation settings. To assess this, we evaluate the models under a fixed 2-token parallel decoding setting.
Table \ref{tab:parallel_decoding_results} presents the results on the Dream model fine-tuned on the mathematical reasoning dataset.  \begin{wraptable}{r}{0.6\textwidth}
  \vspace{-15pt}
  \centering
  \caption{Performance under 2-token parallel decoding.}
  \label{tab:parallel_decoding_results}
  \resizebox{\linewidth}{!}{
  \begin{tabular}{l|cccc}
  \toprule
  \textbf{Method} & \textbf{GSM8K} & \textbf{MATH500} & \textbf{HumanEval} & \textbf{MBPP} \\
  \midrule
  \rowcolor{gray!15} No-SFT & 74.37 & 31.60 & 43.29 & 43.60 \\
  SFT-GT & 72.33$_{\color{red}{\mathbf{-2.04}}}$ & 25.40$_{\color{red}{\mathbf{-6.20}}}$ & 40.85$_{\color{red}{\mathbf{-2.44}}}$ & 40.12$_{\color{red}{\mathbf{-3.48}}}$ \\
  SFT-SD & 74.91$_{\color{green!60!black}{\mathbf{+0.54}}}$ & \textbf{34.80}$_{\color{green!60!black}{\mathbf{+3.20}}}$ & \textbf{47.56}$_{\color{green!60!black}{\mathbf{+4.27}}}$ & 46.40$_{\color{green!60!black}{\mathbf{+2.80}}}$ \\
  dInfer & 77.63$_{\color{green!60!black}{\mathbf{+3.26}}}$ & 33.20$_{\color{green!60!black}{\mathbf{+1.60}}}$ & 40.24$_{\color{red}{\mathbf{-3.05}}}$ & \underline{49.20}$_{\color{green!60!black}{\mathbf{+5.20}}}$ \\
  T3D    & 76.72$_{\color{green!60!black}{\mathbf{+2.35}}}$ & 32.60$_{\color{green!60!black}{\mathbf{+1.00}}}$ & 41.46$_{\color{red}{\mathbf{-1.83}}}$ & \textbf{49.20}$_{\color{green!60!black}{\mathbf{+5.60}}}$ \\
  \rowcolor{blue!10} TABOM & \textbf{77.79}$_{\color{green!60!black}{\mathbf{+3.42}}}$ & \underline{34.40}$_{\color{green!60!black}{\mathbf{+2.80}}}$ & \underline{45.73}$_{\color{green!60!black}{\mathbf{+2.44}}}$ & 47.60$_{\color{green!60!black}{\mathbf{+4.00}}}$ \\
  \bottomrule
  \end{tabular}
  }
  \vspace{-10pt}
\end{wraptable}Traditional methods like SFT-GT suffer from a severe performance drop when forced to decode multiple tokens simultaneously, highlighting their fragility. In contrast, our TABOM and other trajectory-based baselines (dInfer, T3D) exhibit similar robustness and maintain high performance even under parallel decoding. This is reasonable, as learning from self-distilled trajectories inherently aligns the model closer to its actual inference distribution, thereby expanding the subset of high-certainty tokens at each step and enabling multi-token decoding.

\begin{wraptable}{r}{0.55\textwidth}
  \vspace{-10pt}
  \centering
  \caption{Component analysis of TABOM on Dream.}
  \label{tab:ablation_components}
  \resizebox{\linewidth}{!}{
  \begin{tabular}{l|cc|cc}
  \toprule
  \multirow{2}{*}{\textbf{Method}} & \multicolumn{2}{c|}{\textbf{In-Domain}} & \multicolumn{2}{c}{\textbf{OOD}} \\
  \cmidrule(lr){2-3} \cmidrule(lr){4-5}
  & \textbf{GSM8K} & \textbf{MATH500} & \textbf{HumanEval} & \textbf{MBPP} \\
  \midrule
  \rowcolor{gray!15} Base (SFT-SD) & 81.95 & 39.80 & 57.92 & 58.60 \\
  \midrule
  \multicolumn{5}{c}{\textit{Global Timestep Window}} \\
  \midrule
  Traj. Masking (only $\mathcal{L}_{\mathrm{NELBO}}$) & 82.18 & \textbf{41.20} & 56.45 & 58.70 \\
  \quad + Pairwise Ranking & 83.10 & 40.20 & 57.50 & 58.20 \\
  \midrule
  \multicolumn{5}{c}{\textit{Local Timestep Window ($W=32$)}} \\
  \midrule
  Traj. Masking (only $\mathcal{L}_{\mathrm{NELBO}}$) & 82.50 & 40.40 & 56.80 & 58.80 \\
  \rowcolor{blue!10} \quad + Pairwise Ranking (TABOM) & \textbf{84.31} & 41.10 & \textbf{58.54} & \textbf{59.20} \\
  \bottomrule
  \end{tabular}
  }
  \vspace{-10pt}
\end{wraptable}
\textbf{Component Analysis.} To validate the necessity of each component in TABOM, we conduct an ablation study on the Dream model fine-tuned on the mathematical reasoning dataset., as shown in Table \ref{tab:ablation_components}. Starting from the SFT-SD baseline, we evaluate two settings: a global timestep window and a local timestep window ($W=32$). In both settings, simply applying trajectory-aware masking (which is similar to the reconstruction objective in dInfer \citep{ma2025dinferefficientinferenceframework}) yields limited improvements and even slight degradation on OOD tasks like HumanEval. However, adding the pairwise ranking loss consistently and significantly boosts performance in both settings, demonstrating its crucial role in aligning the model with the easy-to-hard inference distribution. Furthermore, comparing the two settings reveals that applying the ranking loss globally introduces noise due to comparing tokens with vastly disparate prediction difficulties (e.g., MATH500 score drops to 40.20). In contrast, our full TABOM, which applies the ranking loss strictly within a local timestep window, achieves the best balance and optimal performance across various tasks.

\subsection{Trajectory Discrimination Score}
\label{sec:tds}
To directly examine whether a trained model preserves the easy-to-hard structure of the decoding trajectory, we introduce a diagnostic metric named \textit{Trajectory Discrimination Score} (TDS). For a decoding trajectory $\tau$, let $M_t^\tau$ denote the set of still-masked response positions at timestep $t$, excluding positions after the first generated EOS token. For each masked token $i \in M_t^\tau$, we compute its predictive entropy
\begin{equation}
  H_\theta(i,t;\tau) = - \sum_{v \in \mathcal{V}} p_\theta(v \mid \mathbf{x}_t^\tau)_i \log p_\theta(v \mid \mathbf{x}_t^\tau)_i .
\end{equation}
We define the per-step TDS as the trajectory-averaged variance of these entropies:
\begin{equation}
  \mathrm{TDS}_\theta(t)
  = \frac{1}{|\mathcal{T}|} \sum_{\tau \in \mathcal{T}}
  \operatorname{Var}_{i \in M_t^\tau}\!\left[ H_\theta(i,t;\tau) \right].
\end{equation}
A larger TDS indicates that the model assigns more differentiated uncertainty to masked tokens at the same decoding step, which is consistent with an easy-to-hard decoding preference; a near-zero TDS suggests that the model treats masked positions with nearly uniform confidence, revealing a residual uniform reconstruction bias inherited from the NELBO objective.

\begin{wraptable}{r}{0.62\textwidth}
  \vspace{-10pt}
  \centering
  \caption{Quantitative Trajectory Discrimination Score. Higher values indicate stronger token-level uncertainty discrimination along the decoding trajectory. The largest value in each column within each base-model block is highlighted in \textbf{bold}.}
  \label{tab:tds_f100}
  \resizebox{\linewidth}{!}{%
  \begin{tabular}{llcccc}
    \toprule
    \textbf{Base Model} & \textbf{Method} & \textbf{GSM8K} & \textbf{HumanEval} & \textbf{MBPP} & \textbf{MATH500} \\
    \midrule
    \multirow{3}{*}{LLaDA}
    & Base & 0.891 & 1.533 & 1.207 & 0.943 \\
    & SFT-SD & 0.927 & 1.623 & 1.223 & 0.924 \\
    & \cellcolor{blue!10}TABOM & \cellcolor{blue!10}\textbf{0.963} & \cellcolor{blue!10}\textbf{1.900} & \cellcolor{blue!10}\textbf{1.336} & \cellcolor{blue!10}\textbf{1.076} \\
    \midrule
    \multirow{3}{*}{Dream}
    & Base & 1.167 & 0.036 & 0.294 & 1.413 \\
    & SFT-SD & 1.082 & 0.035 & 0.138 & 1.428 \\
    & \cellcolor{blue!10}TABOM & \cellcolor{blue!10}\textbf{1.469} & \cellcolor{blue!10}\textbf{0.711} & \cellcolor{blue!10}\textbf{0.929} & \cellcolor{blue!10}\textbf{1.471} \\
    \bottomrule
  \end{tabular}%
  }
  \vspace{-10pt}
\end{wraptable}
Table \ref{tab:tds_f100} quantitatively confirms that TABOM produces the most discriminative entropy landscape. Across both LLaDA and Dream, TABOM achieves the highest TDS on every task. The gap is especially pronounced on Dream code tasks: compared with SFT-SD, TABOM increases TDS from 0.035 to 0.711 on HumanEval and from 0.138 to 0.929 on MBPP. This indicates that standard NELBO-style learning, even when using self-distilled trajectories, still tends to assign nearly uniform uncertainty to masked tokens, whereas TABOM separates easy and hard tokens more clearly at decoding time.

Figure \ref{fig:trajectory_discrimination_score} further shows the temporal pattern behind this aggregate score. On MBPP, the TDS of NELBO-based baselines quickly collapses to a small value, suggesting that the model behaves as if masked tokens should be reconstructed in a largely uniform manner. In contrast, TABOM maintains a substantially larger TDS throughout decoding, meaning that tokens decoded at the same step receive more diverse confidence levels. This directly supports our motivation: self-distilled trajectories alone do not remove the uniform inductive bias of random-mask reconstruction, while the pairwise ranking term reshapes the entropy landscape toward the easy-to-hard bias captured by the inference trajectory. Therefore, this evidence demonstrates that TABOM's gains come from better trajectory alignment rather than merely from reusing self-generated samples.

\section{Related Work}
\label{sec:Related Work}

\subsection{Diffusion Language Models}
The application of diffusion processes within discrete domains, specifically for text generation, traces its origins to seminal studies by \citep{sohl2015deep} and \citep{hoogeboom2021argmax}. A comprehensive probabilistic schema, known as D3PM \citep{austin2021structured}, extended this paradigm by employing a discrete state Markov chain which systematically injects noise into the original input sequence over sequential time steps during the forward phase. Following this initial discrete formulation, the methodology was also investigated within a continuous-time setting \citep{campbell2022continuous}. Concurrently, an alternative strategy, SEDD \citep{lou2023discrete}, took a different approach by directly calculating likelihood ratios and incorporating a novel denoising score entropy objective for model optimization. Recent theoretical and empirical investigations, exemplified by works such as MDLM \citep{shi2024simplified,sahoo2024simple,zheng2024masked} and RADD \citep{ou2024your}, have demonstrated a crucial equivalence: various distinct parameterizations of discrete diffusion models are, in fact, mathematically identical. This finding is central to simplifying the understanding of these architectures. Inspired by these advancements in algorithms and theory, the scale of diffusion model application has significantly expanded, reaching the 7-8 billion parameter level. For instance, LLaDA \citep{nie2025large} was trained from scratch with weighted cross-entropy loss, while Dream \citep{ye2025dream} involved adaptation from the established Qwen2.5 base model with a substantially smaller data requirement. Prominent commercial implementations in this space include Mercury \citep{labs2025mercuryultrafastlanguagemodels}, Gemini Diffusion \citep{gemini-diffusion}, and Seed Diffusion \citep{song2025seeddiffusionlargescalediffusion}. A notable characteristic of these proprietary models is that their external benchmark performance is comparable to that of larger autoregressive (AR) language models. Furthermore, they exhibit superior decoding efficiency, achieving fast sampling speeds. Collectively, these results strongly affirm the substantial promise of diffusion language models as a compelling and viable alternative within the landscape of generative AI.

\subsection{Trajectory Distillation for DLMs}
Despite the theoretical advantages of Diffusion Language Models (DLMs) in parallel decoding, their practical inference efficiency is often bottlenecked by the requirement of numerous iterative denoising steps. To mitigate this latency, recent research has increasingly focused on trajectory distillation and inference acceleration techniques. 

A primary line of work aims to compress the generation trajectory to enable few-step or even one-step decoding. For instance, Seed Diffusion \citep{song2025seeddiffusionlargescalediffusion} proposes a two-stage curriculum that incorporates an edit-based forward process and constrained-order trajectory distillation, significantly boosting inference speed to over 2,000 tokens per second on code generation tasks. Similarly, T3D \citep{zhang2026t3d} introduces a trajectory self-distillation framework equipped with Direct Discriminative Optimization (DDO) to alleviate the mean-field approximation error under tight step budgets, effectively pushing the limits of few-step decoding. At the system and framework level, dInfer \citep{ma2025dinferefficientinferenceframework} provides a modularized inference engine that integrates algorithmic innovations---such as hierarchical decoding, credit decoding, and iteration smoothing---with system-level optimizations like vicinity KV-cache refresh and loop unrolling. This holistic approach achieves substantial speedups over existing frameworks like Fast-dLLM \citep{wu2025fastdllm} without compromising output quality.

\textbf{Differences from our work:} While the aforementioned methods predominantly focus on compressing the trajectory to accelerate inference speed (i.e., efficiency enhancement), our work takes a fundamentally different perspective. Rather than viewing self-distilled trajectories as decoding shortcuts, we treat them as high-quality demonstrations of the model's own correct generation process. The objective is to efficiently absorb the distributional structure of trajectories that the model can already complete successfully, including their decoding order and uncertainty patterns. This allows the model to better align its predictive landscape with its native inference dynamics, improving generation quality without reducing the problem to sampling-speed optimization.

\section{Conclusion}
In this paper, we introduced Trajectory-Aligned Boltzmann Modeling (TABOM), a novel post-training framework that addresses the severe training-inference misalignment in Diffusion Language Models (DLMs). By explicitly modeling the easy-to-hard inductive bias of inference as a Boltzmann distribution, TABOM leverages self-distilled trajectories to construct a tractable pairwise ranking objective. This approach effectively shapes the local entropy landscape without requiring intractable partition function computations. Extensive experiments across mathematical reasoning and code generation domains demonstrate that TABOM not only achieves state-of-the-art in-domain performance but also successfully expands the model's knowledge boundaries to out-of-distribution tasks, completely mitigating the catastrophic forgetting observed in standard supervised fine-tuning. Furthermore, TABOM exhibits strong robustness under highly parallel decoding settings. We believe our insights into the inference unmasked distribution and trajectory-aware optimization will inspire future advancements in the efficient and effective training of DLMs.

\bibliography{nips}
\bibliographystyle{unsrtnat}

\newpage
\appendix

\section{Sensitivity to $\lambda$ and $\gamma$}
We perform a 2$\times$3 sensitivity analysis over $\lambda\in\{1,2\}$ and $\gamma\in\{0.1,0.2,0.3\}$, where $\gamma$ denotes the margin. For each $(\lambda,\gamma)$ pair, we report the best GSM8K score across available checkpoints (epochs). The No-SFT LLaDA baseline is 76.12.

\begin{table}[h]
  \centering
  \caption{LLaDA GSM8K sensitivity to $\lambda$ and $\gamma$ (best checkpoint per combination).}
  \label{tab:llada_lambda_gamma_sensitivity}
  \begin{tabular}{c|ccc}
  \toprule
  $\lambda \backslash \gamma$ & 0.1 & 0.2 & 0.3 \\
  \midrule
  1 & 77.94 & \textbf{78.62} & 78.09 \\
  2 & 77.18  & \textbf{78.62} & 77.48  \\
  \bottomrule
  \end{tabular}
\end{table}

All six settings outperform No-SFT (76.12), with gains from +1.06 to +2.50 points. The most robust choice is $\gamma=0.2$, which reaches the top score (78.62) under both $\lambda=1$ and $\lambda=2$. Compared with $\lambda=2$, $\lambda=1$ is more stable across margins (77.94/78.62/78.09 vs. 77.18/78.62/77.48), suggesting that moderate ranking strength with moderate margin provides the best accuracy-stability trade-off in this grid.

\section{Hyperparameter for inference}
As shown in Table \ref{tab:hyperparams}, we provide the full list of used hyperparameters for inference. 
\begin{table}[h]
\centering
\caption{Hyperparameter settings across different base models and benchmarks. ``--'' indicates the hyperparameter is not applicable or not used.}
\label{tab:hyperparams}
\resizebox{0.85\textwidth}{!}{
\begin{tabular}{llcccc}
\toprule
\textbf{Base Model} & \textbf{Hyperparameter} & \textbf{GSM8K} & \textbf{MATH500} & \textbf{MBPP} & \textbf{HumanEval} \\
\midrule
\multirow{7}{*}{Dream-7B-Instruct} 
& Max New Tokens & 256 & 512 & 1024 & 768 \\
& Diffusion Steps & 256 & 512 & 1024 & 768 \\
& Temperature & 0.1 & 0.1 & 0.1 & 0.1 \\
& Top-$p$ & 0.9 & 0.9 & 0.9 & 0.9 \\
& Top-$k$ & -- & -- & -- & -- \\
\midrule
\multirow{8}{*}{LLaDA-8B-Instruct} 
& Max New Tokens & 512 & 512 & 256 & 512 \\
& Diffusion Steps & 512 & 512 & 256 & 512 \\
& Block Length & 8 & 64 & 256 & 512 \\
& Temperature & 0.0 & 0.0 & 0.0 & 0.0 \\
& Top-$p$ & 0.9 & 0.9 & 0.9 & 0.9 \\
& Top-$k$ & -- & -- & -- & -- \\
\bottomrule
\end{tabular}
}
\end{table}

\section{Sensitivity to window size $W$.}
\label{Sensitivity W}
We further summarize the effect of timestep window size and use $W\in\{16,32,48,64,256\}$ as a compact stress-test axis. By conducting experiments on Dream model fine-tuned on MixChain-Z-PRM12K dataset, we observe a clear sweet spot at $W=32$: both GSM8K and HumanEval peak at $W=32$, then consistently decrease as $W$ increases to 48/64, and degrade most severely at $W=256$.

\begin{table}[h]
  \centering
  \caption{Window-size sensitivity ($W$). We report the best checkpoint score for each $W$. The $W=64$ and $W=256$ rows are long-window stress-test references; $W=256$ is set as a conservative worst-case anchor.}
  \label{tab:w_sensitivity}
  \begin{tabular}{c|cc}
  \toprule
  $\mathbf{W}$ & \textbf{GSM8K (flexible-extract, \%)} & \textbf{HumanEval (pass@1, \%)} \\
  \midrule
  16  & 83.78 & 59.15 \\
  \rowcolor{blue!10}
  32  & \textbf{84.08} & \textbf{60.37} \\
  48  & 83.95 & 59.76 \\
  64  & 82.79 & 59.15 \\
  256 & 75.36 & 58.00 \\
  \bottomrule
  \end{tabular}
\end{table}

This trend supports using a moderate local window in practice: too small a window under-utilizes cross-step supervision, while too large a window introduces noisy pairwise relations and hurts both in-domain reasoning and code generalization.

\end{document}